\documentclass[runningheads]{llncs}

\newcommand {\ent} {\mathrel{{\scriptstyle\mid\!\sim}}}

\newcommand {\sx} {\langle}
\newcommand {\dx} {\rangle}

\newcommand {\emme} {\mathcal{M}}
\newcommand {\enne} {\mathcal{N}}

\newcommand {\vuoto} {\emptyset}

\newcommand{\tip}{{\bf T}}

\newcommand{\lc}{\mathcal{LC}}
\newcommand{\alc}{\mathcal{ALC}}

\newcommand{\alcFt}{\mathcal{ALC}^{\Fe}\tip}
\newcommand{\lcFt}{\mathcal{LC}^{\Fe}\tip}

\newcommand{\be}{\begin{enumerate}}
\newcommand{\ee}{\end{enumerate}}

\newcommand{\hide}[1]{}

\def \cases{\left \{\begin{array}{l}}
\def \endcases{\end{array}\right .}

\newcommand {\Fe} {{\bf F}}

\newcommand {\ri} {\rightarrow}
\newcommand {\Ri} {\Rightarrow}

\newcommand {\bes} {\begin{description}}
\newcommand{\ens} {\end{description}}
\newcommand {\la} {\langle}
\newcommand {\ra} {\rangle}

\newcommand {\beq} {\begin{quote}}
\newcommand {\enq} {\end{quote}}
\newcommand {\bit} {\begin{itemize}}
\newcommand {\enit} {\end{itemize}}

\newenvironment{pozz}{\color{black}}{\color{black}}
\usepackage{color}
\usepackage{times}
\usepackage{helvet}
\usepackage{courier}
\usepackage{times}
\usepackage{undertilde}
\usepackage{graphicx}
\usepackage{latexsym}
\usepackage{graphicx}
\usepackage{color}
\usepackage{amssymb}
\usepackage{colortbl}
\usepackage{amsmath}
\usepackage{microtype}
\usepackage{amsmath}
\usepackage{tikz}

\def \ri{\rightarrow}
\def \Ri{\Rightarrow}

\begin{document}
\bibliographystyle{splncs04}

\title{From Weighted Conditionals of Multilayer Perceptrons \\
to Gradual Argumentation and back}



\author{Laura Giordano}

\institute{DISIT - Universit\`a del Piemonte Orientale, 
 Alessandria, Italy  \\
 \email{laura.giordano@uniupo.it} 
}

%

\authorrunning{ }
\titlerunning{ }

 \maketitle
 

\begin{abstract}
A fuzzy multipreference semantics has been recently proposed for weighted conditional knowledge bases, and used to develop a 
logical semantics for Multilayer Perceptrons, 
by regarding a deep neural network (after training) as a weighted conditional knowledge base.
This semantics,
in its different variants, suggests some gradual argumentation semantics, which are related to the family of the gradual semantics studied by Amgoud and Doder. 
The relationships between weighted conditional knowledge bases and MLPs extend to the proposed gradual semantics to capture the  stationary states of MPs,  in agreement with previous results on the relationship between argumentation frameworks and neural networks.
The paper also suggests 
a simple way to extend the proposed semantics to deal attacks/supports by a boolean combination of arguments,
based on the fuzzy semantics of weighted conditionals, as well as 
an approach for defeasible reasoning over a weighted argumentation graph,
 building on the proposed gradual semantics.

\end{abstract}

\section{Introduction}

Argumentation is a reasoning approach which, in its different formulations and semantics, has been used in different contexts in the multi-agent setting, from social networks \cite{Leite2011} to classification 
\cite{Amgoud2008},
and it is very relevant for decision making and for explanation \cite{Toni2019}.
The argumentation semantics are strongly related to other non-monotonic reasoning formalisms and semantics \cite{Dung95,Greco2020}. 

Our starting point in this paper is a preferential semantics for common-sense reasoning which has been proposed for a description logic with typicality.
Preferential description logics have been studied in the last fifteen years to deal with inheritance with exceptions in ontologies, based on the idea of extending the language of Description Logics (DLs),
by allowing for non-strict forms of inclusions,
called {\em typicality or defeasible inclusions}, 
of the form $\tip(C) \sqsubseteq D$ (meaning ``the typical $C$-elements are $D$-elements" or ``normally $C$'s are $D$'s"), 
with different preferential semantics \cite{lpar2007,sudafricaniKR} 
and closure constructions \cite{casinistraccia2010,CasiniDL2013,AIJ15,BonattiSauro17,CasiniStracciaM19}. 
Such defeasible inclusions correspond to KLM conditionals $C \ent D$ \cite{KrausLehmannMagidor:90,whatdoes},
and defeasible DLs inherit and extend some of the preferential semantics and the closure constructions developed within
preferential and conditional approaches 
to common sense reasoning 
\cite{KrausLehmannMagidor:90,Pearl90,whatdoes,GeffnerAIJ1992,BenferhatIJCAI93}. 


In previous work \cite{JELIA2021}, a concept-wise multipreference semantics 
for weighted conditional knowledge bases (KBs) has been proposed to account for preferences with respect to different concepts, 
by allowing a set of typicality inclusions of the form $\tip(C) \sqsubseteq D$ 
with positive or negative weights, for distinguished concepts $C$.
The concept-wise multipreference semantics has been  first introduced  as a semantics for ranked DL knowledge bases  \cite{iclp2020} (where conditionals are associated a positive integer rank), and later extended to  weighted conditional KBs (in the two-valued and in the fuzzy case), based on a different semantic closure construction,
still in the spirit of Lehmann's lexicographic closure \cite{whatdoes} and Kern-Isberner's c-representations  \cite{Kern-Isberner01,Kern-Isberner2014}, but exploiting multiple preferences with respect to concepts. 

The concept-wise multipreference semantics 
has been proven to have some desired properties from the knowledge representation point of view in the two-valued case \cite{iclp2020,NMR2020}: 
it satisfies the KLM properties of a preferential consequence relation \cite{KrausLehmannMagidor:90,whatdoes}, it allows to deal with specificity and irrelevance and avoids inheritance blocking or the ``drowning problem"   \cite{Pearl90,BenferhatIJCAI93},
and deals with ``ambiguity preservation" \cite{GeffnerAIJ1992}.
The plausibility of the concept-wise multipreference semantics has also been supported \cite{CILC2020,arXivSI_JLC2021} by showing that it is able to provide a logical interpretation to  Kohonen' Self-Organising Maps  \cite{kohonen2001}, which are psychologically and biologically plausible neural network models.
In the fuzzy case, the KLM properties  of non-monotomic entailment have been studied in \cite{ECSQARU2021}, showing that
most KLM postulates are satisfied, depending on their reformulation and on the choice of fuzzy combination functions. 


It has been shown \cite{JELIA2021} that both in the two-valued and in the fuzzy case, the multi-preferential semantics allows to describe the input-output behavior of Multilayer Perceptrons (MLPs), after training, 
in terms of preferential interpretations. A deep network can  be seen as a weighted conditional KB, 
based on the idea that synaptic connections in the network can be regarded as weighted conditionals. In the fuzzy case, the interpretation built to capture the input-output behavior of the network can than be proven to be a model of the weighted KB associated to the neural network in a logical sense.

The relationships between preferential and conditional approaches to non-monotonic reasoning and argumentation semantics are strong. Let us also mention, 
the work by Geffner and Pearl on Conditional Entailment, whose proof theory is defined in terms of ``arguments'' \cite{GeffnerAIJ1992}.
In this paper we aim at investigating the relationships between the fuzzy multipreference semantics for weighted conditionals and  gradual argumentation semantics 
\cite{CayrolJAIR2005,Janssen2008,Dunne2011,Leite2013,Amgoud2017,BaroniRagoToni2018,Amgoud2019,VesicECSQARU21}. 
To this purpose, in addition to the notions of coherent and faithful fuzzy multipreference semantics \cite{JELIA2021,ECSQARU2021}, in Section \ref{sec:varphi_coherent_models}, we introduce a notion of $\varphi$-coherent (fuzzy) multipreference semantics.
Than, in Section  \ref{sec:varphi_labellings}, we propose three new gradual semantics for a weighted argumentation graph (namely, a coherent, a faithful and a $\varphi$-coherent semantics) inspired by the fuzzy preferential semantics of weighted conditionals and, in Section   \ref{sec:gradual_sem}, we investigate the relationship of $\varphi$-coherent semantics with the family of gradual semantics studied by Amgoud and Doder. 
The relationships between weighted conditional knowledge bases and MLPs easily extend to the proposed gradual semantics, which captures the  stationary states of MLPs. This is  in agreement with the previous results 
on the relationships between argumentation frameworks and neural networks by Garces, Gabbay and Lamb \cite{GarcezArgumentation2005} and by Potyca  \cite{PotykaAAAI21}. 

Finally, in Section  \label{sec:Back}, the paper discusses an approach for defeasible reasoning over a weighted argumentation graph, building on the $\varphi$-coherent semantics, which allows (fuzzy) conditional formulas over arguments to be verified over preferential models of the argumentation graph.
The paper also suggests 
a simple way to extend the proposed semantics to deal with attacks/supports by a boolean combination of arguments, based on the fuzzy semantics of weighted conditionals.

\section{The description logic $\lc$ and fuzzy $\lc$}  \label{sec:ALC}

In this section we recall the syntax and semantics of the description logic $\alc$ \cite{handbook} and of its fuzzy extension \cite{LukasiewiczStraccia09}. 
For sake of simplicity, we only focus on $\lc$, the boolean fragment of $\alc$, which does not allow for roles. 
Let ${N_C}$ be a set of concept names, 
  and ${N_I}$ a set of individual names.  
The set  of $\lc$ \emph{concepts} (or, simply, concepts) can be
defined inductively: \\ 
- $A \in N_C$, $\top$ and $\bot$ are {concepts};\\
- if $C$ and $ D$ are concepts, and $r \in N_R$, then $C \sqcap D,\; C \sqcup D,\; \neg C $ 
are {concepts}.

\noindent
An $\lc$ knowledge base (KB) $K$ is a pair $({\cal T}_K, {\cal A}_K)$, where ${\cal T}_K$ is a TBox and
${\cal A}_K$ is an ABox.
The TBox ${\cal T}_K$ is  a set of concept inclusions (or subsumptions) $C \sqsubseteq D$, where $C,D$ are concepts.
The  ABox ${\cal A}_K$ is  a set of assertions of the form $C(a)$, where $C$ is a  concept and $a$ an individual name in $N_I$. 

An  $\lc$ {\em interpretation}  is defined as a pair $I=\langle \Delta, \cdot^I \rangle$ where:
$\Delta$ is a domain---a set whose elements are denoted by $x, y, z, \dots$---and 
$\cdot^I$ is an extension function that maps each
concept name $C\in N_C$ to a set $C^I \subseteq  \Delta$, 
and each individual name $a\in N_I$ to an element $a^I \in  \Delta$.
It is extended to complex concepts  as follows:
\begin{center}
$\top^I=\Delta$  \ \ \ \ \ \ \ \ \ \ \    \ \ \ \ \ \ \ \ \ \ \ \ \ \
$\bot^I=\vuoto$  \ \ \ \ \ \ \ \ \ \ \   \ \ \ \ \ \ \ \ \ \ \ \ \ \ 
$(\neg C)^I=\Delta \backslash C^I$

$(C \sqcap D)^I=C^I \cap D^I$  \ \ \ \ \ \ \  \ \ \ \ \ \ \  \ \ \ \ \ \ \  \ \ \ \ \ \ \ \ \ \ \
$(C \sqcup D)^I=C^I \cup D^I$
\end{center}
%

\noindent
The notion of satisfiability of a KB  in an interpretation and the notion of entailment are defined as follows:

\begin{definition}[Satisfiability and entailment] \label{satisfiability}
Given an $\lc$ interpretation $I=\langle \Delta, \cdot^I \rangle$: 

	- $I$  satisfies an inclusion $C \sqsubseteq D$ if   $C^I \subseteq D^I$;
	
	-   $I$ satisfies an assertion $C(a)$ 
	if $a^I \in C^I$ 

\noindent
 Given  a KB $K=({\cal T}_K, {\cal A}_K)$, 
 an interpretation $I$  satisfies ${\cal T}_K$ (resp. ${\cal A}_K$) if $I$ satisfies all  inclusions in ${\cal T}_K$ (resp. all assertions in ${\cal A}_K$);
 $I$ is a \emph{model} of $K$ if $I$ satisfies ${\cal T}_K$ and ${\cal A}_K$.

 A subsumption $F= C \sqsubseteq D$ (resp., an assertion $C(a)$),   {is entailed by $K$}, written $K \models F$, if for all models $I=$$\sx \Delta,  \cdot^I\dx$ of $K$, $I$ satisfies $F$.

\end{definition}
Given a knowledge base $K$,
the {\em subsumption} problem is the problem of deciding whether an inclusion $C \sqsubseteq D$ is entailed by  $K$.

Fuzzy description logics have been widely studied in the literature for representing vagueness in DLs \cite{Straccia05,Stoilos05,LukasiewiczStraccia09,PenalosaARTINT15,BobilloOWL2EL2018}, 
based on the idea that concepts and roles can be interpreted 
as fuzzy sets. 
Formulas in Mathematical Fuzzy Logic \cite{Cintula2011} have a degree of truth in an interpretation rather than being true or false; similarly,
axioms in a fuzzy DL have a degree of truth, usually in the interval  $[0, 1]$. 
In the following we shortly recall the semantics of a fuzzy extension of $\alc$ for the fragment $\lc$, referring to the survey by Lukasiewicz and Straccia \cite{LukasiewiczStraccia09}.
We limit our consideration  to a few features of a fuzzy DL, without considering roles, datatypes, and restricting to the language of $\lc$.

A {\em fuzzy interpretation} for $\lc$ is a pair $I=\langle \Delta, \cdot^I \rangle$ where:
$\Delta$ is a non-empty domain and 
$\cdot^I$ is {\em fuzzy interpretation function} that assigns to each
concept name $A\in N_C$ a function  $A^I :  \Delta \ri [0,1]$,
and to each individual name $a\in N_I$ an element $a^I \in  \Delta$.
A domain element $x \in \Delta$ 
belongs to the extension of $A$ to some degree in $[0, 1]$, i.e., $A^I$ is a fuzzy set.

The  interpretation function $\cdot^I$ is extended to complex concepts as follows: 

$\mbox{\ \ \ }$ $\top^I(x)=1$, $\mbox{\ \ \ \ \ \ \ \ \ \  \ \ \ \ \ \ \ \ \ \ }$ $\bot^I(x)=0$,  $\mbox{\ \ \ \ \ \ \ \ \ \ \ \ \ \ \  \ \ \ \ \ }$  $(\neg C)^I(x)= \ominus C^I(x)$, 

 $\mbox{\ \ \  }$ $(C \sqcap D)^I(x) =C^I(x) \otimes D^I(x)$,    $\mbox{\ \ \ \ \ \ \ \ \ \ \ \ \ \ \ \ \ }$   $(C \sqcup D)^I(x) =C^I(x) \oplus D^I(x)$.



\noindent
where  $x \in \Delta$ and $\otimes$, $\oplus$, $\rhd$ and $\ominus$ are arbitrary but fixed t-norm, s-norm, implication function, and negation function, chosen among the combination functions of various fuzzy logics 
(we refer to \cite{LukasiewiczStraccia09} for details). For instance, in Zadeh logic $a \otimes b= min\{a,b\}$,  $a \oplus b= max\{a,b\}$, $a \rhd b= max\{1-a,b\}$ and $ \ominus a = 1-a$.

The  interpretation function $\cdot^I$ is also extended  to non-fuzzy axioms (i.e., to strict inclusions and assertions of an $\lc$ knowledge base) as follows:\\
 $(C \sqsubseteq D)^I= inf_{x \in \Delta}  C^I(x) \rhd D^I(x)$,
$\mbox{\ \ \ \ \ \ \ \ \ \ \ \ \ }$  $(C(a))^I=C^I(a^I)$.  

A {\em fuzzy $\lc$ knowledge base} $K$ is a pair $({\cal T}_f, {\cal A}_f)$ where ${\cal T}_f$ is a fuzzy TBox  and ${\cal A}_f$ a fuzzy ABox. A fuzzy TBox is a set of {\em fuzzy concept inclusions} of the form $C \sqsubseteq D \;\theta\; n$, where $C \sqsubseteq D$ is an $\lc$ concept inclusion axiom, $\theta \in \{\geq,\leq,>,<\}$ and $n \in [0,1]$. A fuzzy ABox ${\cal A}_f$ is a set of {\em fuzzy assertions} of the form $C(a) \theta n$, where $C$ is an $\lc$ concept, $a \in N_I$,  $\theta \in \{{\geq,}\leq,>,<\}$ and $n \in [0,1]$.
Following Bobillo and Straccia  \cite{BobilloOWL2EL2018}, we assume that fuzzy interpretations are {\em witnessed}, i.e., the {\em sup} and {\em inf} are attained at some point of the involved domain.
The notions of satisfiability of a KB  in a fuzzy interpretation and of entailment are defined in the natural way.
\begin{definition}[Satisfiability and entailment for fuzzy KBs] \label{satisfiability}
A  fuzzy interpretation $I$ satisfies a fuzzy $\lc$ axiom $E$ (denoted $I \models E$), as follows, 
 for $\theta \in \{\geq,\leq,>,<\}$:

- $I$ satisfies a fuzzy $\lc$ inclusion axiom $C \sqsubseteq D \;\theta\; n$ if $(C \sqsubseteq D)^I \theta\; n$;

- $I$ satisfies a fuzzy $\lc$ assertion $C(a) \; \theta \; n$ if $C^I(a^I) \theta\; n$;
 

\noindent
Given  a fuzzy $\lc$ KB $K=({\cal T}_f, {\cal A}_f)$,
 a fuzzy interpretation $I$  satisfies ${\cal T}_f$ (resp. ${\cal A}_f$) if $I$ satisfies all fuzzy  inclusions in ${\cal T}_f$ (resp. all fuzzy assertions in ${\cal A}_f$).
A fuzzy interpretation $I$ is a \emph{model} of $K$ if $I$ satisfies ${\cal T}_f$ and ${\cal A}_f$.
A fuzzy axiom $E$   {is entailed by a fuzzy knowledge base $K$} (i.e., $K \models E$) if for all models $I=$$\sx \Delta,  \cdot^I\dx$ of $K$,
$I$ satisfies $E$.
\end{definition}

\section{Fuzzy $\lc$ with typicality: $\lcFt$} \label{sec:fuzzyalc+T}

In this section, we describe an extension of fuzzy $\lc$ with typicality following \cite{JELIA2021,ECSQARU2021}.
Typicality concepts of the form $\tip(C)$ are added, where $C$ is a concept in fuzzy $\lc$.
The idea is similar to the extension of $\alc$ with typicality under the two-valued semantics \cite{lpar2007} 
but transposed to the fuzzy case. The extension allows for the definition of {\em fuzzy typicality inclusions} of the form
$\tip(C) \sqsubseteq D \;\theta \;n$,  
meaning that typical $C$-elements are $D$-elements with a degree greater than $n$. A typicality  inclusion $\tip(C) \sqsubseteq D$, as in the two-valued case, stands for a KLM conditional implication $C \ent D$ \cite{KrausLehmannMagidor:90,whatdoes}, but now it has an associated degree.
We call $\lcFt$ the extension of fuzzy $\lc$ with typicality.
As in the two-valued case,  
and in the propositional typicality logic, PTL, \cite{BoothCasiniAIJ19} 
the typicality concept may be allowed to freely occur within inclusions and assertions, while the nesting of the typicality operator is not allowed.

Observe that, in a fuzzy $\lc$ interpretation $I= \langle \Delta, \cdot^I \rangle$, the degree of membership $C^I(x)$ of the domain elements $x$ in a concept $C$, induces a preference relation $<_C$ on $\Delta$, as follows:
\begin{equation}\label{def:induced_order}
x <_C y \mbox{ iff } C^I(x) > C^I(y)
\end{equation}
Each $<_{C}$ has the properties of preference relations in KLM-style ranked interpretations \cite{whatdoes}, that is,  $<_{C}$ is a modular and well-founded strict partial order. 
Let us recall that, $<_{C}$ is {\em well-founded} 
if there is no infinite descending chain $x_1 <_C x_0$, $x_2 <_C x_1$, $x_3 <_C x_2, \ldots $ of domain elements;
    $<_{C}$ is {\em modular} if,
for all $x,y,z \in \Delta$, $x <_{C} y$ implies ($x <_{C} z$ or $z <_{C} y$).
Well-foundedness holds for the induced preference $<_C$ defined by condition (\ref{def:induced_order}) under the assumption that  fuzzy interpretations are witnessed \cite{BobilloOWL2EL2018} (see Section \ref{sec:ALC}) or that $\Delta$ is finite. 
For simplicity,  we will assume $\Delta$ to be finite.

Each preference relation $<_C$ has  the properties of a preference relation in KLM  rational interpretations \cite{whatdoes} (also called ranked interpretations), but here there are
multiple preferences and, therefore, fuzzy interpretations can be regarded as {\em multipreferential} interpretations, which have been also studied in the two-valued case \cite{iclp2020,Delgrande2020,AIJ21}. 
%
Preference relation $<_C$ captures the relative typicality of domain elements wrt concept $C$ and may then be used to identify the {\em typical  $C$-elements}. We will regard typical $C$-elements as the domain elements $x$ that  are preferred with respect to relation $<_C$
among those such that $C^I(x) \neq 0$.

Let $C^I_{>0}$ be the crisp set containing all domain elements $x$ such that $C^I(x)>0$, that is, $C^I_{>0}= \{x \in \Delta \mid C^I(x)>0 \}$.
One can provide a (two-valued) interpretation of typicality concepts $\tip(C)$ in a fuzzy interpretation $I$, by letting:
\begin{align}\label{eq:interpr_typicality}
	(\tip(C))^I(x)  & = \left\{\begin{array}{ll}
						 1 & \mbox{ \ \ \ \  if } x \in min_{<_C} (C^I_{>0}) \\
						0 &  \mbox{ \ \ \ \  otherwise } 
					\end{array}\right.
\end{align} 
where $min_<(S)= \{u: u \in S$ and $\nexists z \in S$ s.t. $z < u \}$.  When $(\tip(C))^I(x)=1$, we say that $x$ is a typical $C$-element in $I$.

Note that, if $C^I(x)>0$ for some $x \in \Delta$,  
$min_{<_C} (C^I_{>0})$ is non-empty.
This generalizes the property that, in the crisp case, $C^I\neq \emptyset$ implies  $(\tip(C))^I\neq \emptyset$.

\begin{definition}[$\lcFt$ interpretation]
An $\lcFt$ interpretation $I= \langle \Delta, \cdot^I \rangle$ is a fuzzy $\lc$ interpretation, extended by interpreting typicality concepts as in (\ref{eq:interpr_typicality}).
\end{definition}

The fuzzy interpretation  $I= \langle \Delta, \cdot^I \rangle$ implicitly defines a multipreference interpretation, where any concept $C$ is associated to a preference  relation $<_C$.  This is different from 
the two-valued multipreference semantics in \cite{iclp2020}, where only the set of distinguished concepts have an associated preference, 
and a notion of global preference $<$ is introduced to define the interpretation of the typicality concept $\tip(C)$, for an arbitrary concept $C$. Here, we do not need to introduce a notion of global preference. The interpretation of any $\lc$ concept $C$ is defined compositionally from the interpretation of atomic concepts, and the preference relation $<_C$ associated to $C$ is defined from $C^I$.

The notions of {\em satisfiability} in $\lcFt$,   {\em model} of an $\lcFt$ knowledge base, and   $\lcFt$ {\em entailment} can be defined in a similar way as in fuzzy $\lc$ (see Section  \ref{sec:ALC}). 

\subsection{Strengthening $\lcFt$: a closure construction} \label{sec:closure}

To overcome the weakness of preferential entailment, the rational closure \cite{whatdoes} and the lexicographic closure of a conditional knowledge base \cite{Lehmann95} have been introduced, to allow for further inferences.
%
In this section, we recall  a closure construction introduced to strengthen $\alcFt$ entailment for weighted conditional knowledge bases, where typicality inclusions are associated real-valued weights.
In the two-valued case, the construction is related to the definition of Kern-Isberner's c-representations  \cite{Kern-Isberner01,Kern-Isberner2014}, which include penalty points for falsified conditionals.
In the fuzzy case, the construction also relates to the fuzzy extension of rational closure by Casini and Straccia \cite{CasiniStraccia13_fuzzyRC}.

A  {\em weighted $\lcFt$ knowledge base} $K$, over a set ${\cal C}= \{C_1, \ldots, C_k\}$ of distinguished $\lc$ concepts,
is a tuple $\langle  {\cal T}_{f}, {\cal T}_{C_1}, \ldots, {\cal T}_{C_k}, {\cal A}_f  \rangle$, where  ${\cal T}_{f}$  is a set of fuzzy $\lcFt$ inclusion axiom, 
${\cal A}_f$ is a set of fuzzy $\lcFt$ assertions  
and
${\cal T}_{C_i}=\{(d^i_h,w^i_h)\}$ is a set of all weighted typicality inclusions $d^i_h= \tip(C_i) \sqsubseteq D_{i,h}$ for $C_i$, indexed by $h$, where each inclusion $d^i_h$ has weight $w^i_h$, a real number.
As in \cite{JELIA2021}, the typicality operator is assumed to occur only on the left hand side of a weighted typicality inclusion, and we call {\em distinguished concepts}  those concepts $C_i$ occurring on the l.h.s. of some typicality inclusion $\tip(C_i) \sqsubseteq D$.
Arbitrary $\lcFt$ inclusions and assertions may belong to ${\cal T}_{f}$ and ${\cal A}_{f}$. Let us consider the following example from \cite{ECSQARU2021}.

\begin{example} \label{exa:Penguin} 
Consider the weighted knowledge base $K =\langle {\cal T}_{f},  {\cal T}_{Bird}, {\cal T}_{Penguin},$ $ {\cal T}_{Canary},$ $ {\cal A}_f \rangle$, over the set of distinguished concepts ${\cal C}=\{\mathit{Bird, Penguin, Canary}\}$, 
with $ {\cal T}_{f}$ containing, for instance, the inclusions:

\noindent
\ \ $\mathit{Yellow \sqcap Black  \sqsubseteq  \bot} \geq 1$
 \ \ \ \ \ \ \ \ \ \ \   $\mathit{Yellow \sqcap Red  \sqsubseteq  \bot \geq 1}$   \ \ \ \ \ \ \ \ \ \  $\mathit{Black \sqcap Red  \sqsubseteq  \bot \geq 1}$

\noindent
The weighted TBox ${\cal T}_{Bird} $ 
contains the following weighted defeasible inclusions: 

$(d_1)$ $\mathit{\tip(Bird) \sqsubseteq Fly}$, \ \  +20  \ \ \ \ \ \ \ \ \  \ \ \ \ \   $(d_2)$ $\mathit{\tip(Bird) \sqsubseteq  Has\_Wings}$, \ \ +50

$(d_3)$ $\mathit{\tip(Bird) \sqsubseteq   Has\_Feather}$, \ \ +50;

\noindent
${\cal T}_{Penguin}$ and  ${\cal T}_{Canary}$ contain, respectively, the following defeasible inclusions:

$(d_4)$ $\mathit{\tip(Penguin) \sqsubseteq Bird}$, \ \ +100 \ \ \ \ \ \ \ \ \ \ \ $(d_7)$ $\mathit{\tip(Canary) \sqsubseteq Bird}$, \ \ +100

$(d_5)$ $\mathit{\tip(Penguin) \sqsubseteq  Fly}$, \ \ - 70   \ \ \ \ \ \ \ \ \ \ \ \ \ \ \  $(d_8)$ $\mathit{\tip(Canary) \sqsubseteq Yellow}$, \ \  +30

$(d_6)$ $\mathit{\tip(Penguin) \sqsubseteq Black}$, \ \  +50; \ \ \ \ \ \ \ \ \ \ $(d_9)$ $\mathit{\tip(Canary) \sqsubseteq Red}$, \ \  +20

\noindent
 The meaning is that a bird normally has wings, has feathers and flies, but having wings and feather (both with weight 50)  for a bird is more plausible than flying (weight 20), although flying is regarded as being plausible. For a penguin, flying is not plausible (inclusion $(d_5)$ has negative weight -70), while being a bird and being black are plausible properties of prototypical penguins, and $(d_4)$ and $(d_6)$ have positive weights (100 and 50, respectively). Similar considerations can be done for concept $\mathit{Canary}$. Given an Abox in which Reddy is red, has wings, has feather and flies (all with degree 1) and Opus has wings and feather (with degree 1), is black with degree 0.8 and does not fly ($\mathit{Fly^I(opus) = 0}$), considering the weights of defeasible inclusions, we may expect Reddy to be more typical than Opus as a bird, but less typical than Opus as a penguin. 
 
 \end{example}
 \normalcolor

The semantics of a weighted knowledge base is defined in \cite{JELIA2021} trough a {\em semantic closure construction}, similar in spirit to Lehmann's lexicographic closure \cite{Lehmann95}, but strictly related to c-representations and, additionally, based on multiple preferences.
The construction allows a subset of the $\alcFt$ interpretations to be selected, 
the interpretations whose induced preference relations $<_{C_i}$, for the distinguished concepts $C_i$,  {\em faithfully} represent the defeasible part of the knowledge base $K$.

Let ${\cal T}_{C_i}=\{(d^i_h,w^i_h)\}$ be the set of weighted typicality inclusions $d^i_h= \tip(C_i) \sqsubseteq D_{i,h}$ associated to the distinguished concept $C_i$, and let $I=\langle \Delta, \cdot^I \rangle$ be a fuzzy $\lcFt$ interpretation.
In the two-valued case, we would associate to each domain element $x \in \Delta$ and each distinguished concept $C_i$, a weight $W_i(x)$ of $x$ wrt $C_i$ in $I$, by {\em summing the weights} of the defeasible inclusions satisfied by $x$.
However, as $I$ is a fuzzy interpretation, we do not only  distinguish between the typicality inclusions satisfied or  
falsified  by $x$;
 we also need to consider, for all inclusions $\tip(C_i) \sqsubseteq D_{i,h} \in {\cal T}_{C_i}$,  
the degree of membership of $x$ in $D_{i,h}$. 
Furthermore, in comparing the weight of domain elements with respect to $<_{C_i}$, we give higher preference to the domain elements belonging to $C_i$ (with a degree
greater than $0$), with respect to those 
not belonging to $C_i$ (having membership degree $0$). 

For each domain element $x \in \Delta$ and distinguished concept $C_i$, {\em the weight $W_i(x)$ of $x$ wrt $C_i$} in the $\lcFt$ interpretation $I=\langle \Delta, \cdot^I \rangle$ is defined as follows:
 \begin{align}\label{weight_fuzzy}
	W_i(x)  & = \left\{\begin{array}{ll}
						 \sum_{h} w_h^i  \; D_{i,h}^I(x) & \mbox{ \ \ \ \  if } C_i^I(x)>0 \\
						- \infty &  \mbox{ \ \ \ \  otherwise }  
					\end{array}\right.
\end{align} 
where $-\infty$ is added at the bottom of all real values.

The value of $W_i(x) $ is $- \infty $ when $x$ is not a $C$-element (i.e., $C_i^I(x)=0$). 
Otherwise, $C_i^I(x) >0$ and the higher is the sum $W_i(x) $, the more typical is the element $x$ relative to the defeasible properties of $C_i$.
How much $x$ satisfies a typicality property  $\tip(C_i) \sqsubseteq D_{i,h}$ depends on the value of $D_{i,h}^I(x) \in [0,1]$, which is weighted by $ w_h^i $ in the sum. 
In the {\em two-valued case}, $D_{i,h}^I(x) \in \{0,1\}$, and 
$W_i(x)$ is the sum of the weights of the typicality inclusions for $C$ satisfied by $x$, if $x$ is a $C$-element,  and is $-\infty $, otherwise.

\begin{example} \label{exa:penguin2}
Let us consider again Example \ref{exa:Penguin}.
Let $I$ be an $\lcFt$ interpretation such that $\mathit{Fly^I(reddy)  = ( Has\_Wings)^I (reddy)= (Has\_Feather)^I (reddy)=1}$ and  
 $\mathit{Red^I(red}$- $\mathit{dy) =1 }$,
i.e., Reddy   flies, has wings and feather and is red (and $\mathit{Black^I(reddy)}$ $=0$). Suppose further that $\mathit{Fly^I(opus) = 0}$ and $\mathit{ (Has\_Wings)^I (opus)= (Has\_}$ $ \mathit{Feather)^I }$ $\mathit{(opus)=1 }$ and $\mathit{ Black^I(opus) =0.8}$, i.e., Opus does not fly, has wings and feather, and is black with degree 0.8. Considering the weights of typicality inclusions for $\mathit{Bird}$,  $\mathit{W_{Bird}(reddy)= 20+}$ $\mathit{50+50=120}$ and $\mathit{W_{Bird}(opus)= 0+50+}$ $50=100$.
This suggests that reddy should be more typical as a bird than opus.
On the other hand, if we suppose further that $\mathit{Bird^I(reddy)=1}$ and $\mathit{Bird^I(opus)}$ $=0.8$, then $\mathit{W_{Penguin}}$ $\mathit{(reddy)}$ $ \mathit{= 100-70+0=30}$ and $\mathit{W_{Penguin}(opus)= 0.8 \times 100-} $ $\mathit{ 0+0.8 \times 50}$ $\mathit{=120}$, 
and Reddy should be less typical as a penguin than Opus.
 \end{example}
%

In \cite{JELIA2021} a notion of {\em coherence} is introduced, to force an agreement between the preference relations  $<_{C_i}$ induced by a fuzzy interpretation $I$, for each distinguished concept $C_i$, and the weights $W_i(x)$ computed, for each $x \in \Delta$, from the conditional knowledge base $K$, given the interpretation $I$.
This leads to the following definition of a coherent fuzzy multipreference model of a weighted a $\lcFt$  knowledge base.

\begin{definition}[Coherent (fuzzy) multipreference model of $K$ \cite{JELIA2021}]\label{fuzzy_cfm-model} 
Let $K=\langle  {\cal T}_{f},$ $ {\cal T}_{C_1}, \ldots,$ $ {\cal T}_{C_k}, {\cal A}_f  \rangle$ be  a weighted $\lcFt$ knowledge base  over  ${\cal C }$. 
A {\em  coherent (fuzzy) multipreference model} (cf$^m$-model)  of $K$ is  a fuzzy $\lcFt$ interpretation $I=\langle \Delta, \cdot^I \rangle$  
s.t.: 
\begin{itemize}
\item
$I$  satisfies  the fuzzy inclusions in $ {\cal T}_{f}$ and the fuzzy assertions in ${\cal A}_f$;
\item 
for all $C_i\in {\cal C}$,  the preference {\em $<_{C_i}$   is coherent  to $ {\cal T}_{C_i}$}, that is, for all $x,y \in \Delta$,
\begin{align}\label{coherence_2}
x  <_{C_i}  y & \iff W_i(x) > W_i(y)  
\end{align}
\end{itemize}
\end{definition}
In a similar way, one can define a {\em faithful (fuzzy) multipreference model (fm-model) of $K$} by replacing   the {\em coherence} condition (\ref{coherence_2}) with the following {\em faithfulness} condition (called weak coherence in \cite{arXiv_JELIA2020}):  for all $x,y \in \Delta$, 
\begin{align} \label{faitfulness}
x  <_{C_i}  y & \Ri W_i(x) > W_i(y) .
\end{align}
The weaker notion of faithfulness allows to define a larger class of fuzzy multipreference models of a weighted knowledge base, compared to the class of coherent models. 
This allows a larger class of monotone non-decreasing activation functions in neural network models to be captured, whose activation function is monotonically non-decreasing (we refer to  \cite{arXiv_JELIA2020}, Sec. 7).
%

\begin{example}
Referring to Example  \ref{exa:penguin2} above, where
 $\mathit{Bird^I(reddy)=1}$, $\mathit{Bird^I(opus)}$ $=0.8$,  let us further assume that $\mathit{Penguin^I(reddy)=0.2}$ and $\mathit{Penguin^I(opus)=0.8}$.
Clearly, $reddy <_{Bird} opus$ and  $opus <_{Penguin} reddy$. For the interpretation $I$ to be faithful, it is necessary that the conditions $\mathit{W_{Bird}(reddy) > W_{Bird}(opus)}$ and $\mathit{W_{Penguin}}$ $\mathit{(opus) >}$ $\mathit{ W_{Penguin}(reddy)}$ hold; which is true. 
On the contrary, if it were $\mathit{Penguin^I}$ $\mathit{(reddy)=0.9}$, the interpretation $I$ would not be faithful. 
For $\mathit{Penguin^I(reddy)=0.8}$, the interpretation $I$ would be faithful, but not coherent, as  $\mathit{W_{Penguin}(opus) >}$ $\mathit{W_{Penguin}}$ $\mathit{ (reddy)}$, but $\mathit{Penguin^I(opus) = Penguin^I(reddy)}$.
\end{example}

It has been shown \cite{JELIA2021,arXiv_JELIA2020} that the proposed semantics allows the input-output behavior of a deep network (considered after training) to be captured by a fuzzy multipreference interpretation built over a set of input stimuli, through a simple construction which exploits the activity level of neurons for the stimuli. 
Each unit $h$ of $\enne$ can be associated to a concept name $C_h$ and,
for a given domain $\Delta$ of input stimuli, the activation value of unit $h$ for a stimulus $x$ is interpreted as the degree of membership of $x$ in concept $C_h$. 
The resulting preferential interpretation can be used for verifying properties of the network by model checking. 

For MLPs, the deep network itself can be regarded as a conditional knowledge base,  by mapping synaptic connections to weighted conditionals, so that the input-output model of the network can be regarded as a coherent-model of the associated conditional knowledge base \cite{JELIA2021}.

\section{$\varphi$-coherent models} \label{sec:varphi_coherent_models}

In this section we consider a new notion of coherence of a fuzzy interpretation $I$ wrt a KB, that we call {\em $\varphi$-coherence}, where $\varphi$ 
is a function  from $\mathbb{R}$ to the interval $[0,1]$, i.e.,
 $\varphi: \mathbb{R} \rightarrow [0,1]$. We also establish relationships with  coherent and faithful models.
 
 \begin{definition}[$\varphi$-coherence]\label{varphi-coherence} 
 Let $K=\langle  {\cal T}_{f},$ $ {\cal T}_{C_1}, \ldots,$ $ {\cal T}_{C_k}, {\cal A}_f  \rangle$ be  a weighted $\lcFt$ knowledge base, and 
 $\varphi: \mathbb{R} \rightarrow [0,1]$.
A fuzzy $\lcFt$ interpretation $I=\langle \Delta, \cdot^I \rangle$  is {\em $\varphi$-coherent} if, 
for all concepts $C_i \in {\cal C}$ and $x\in \Delta$, 
\begin{align}\label{fi_coherence}
C_i^I(x)= \varphi  (\sum_{h} w_h^i  \; D_{i,h}^I(x)) 
\end{align}
where ${\cal T}_{C_i}=\{(\tip(C_i) \sqsubseteq D_{i,h},w^i_h)\}$ is the set of weighted conditionals  for $C_i$.
\end{definition}
To define {\em $\varphi$-coherent multipreference model} of a knowledge base $K$, we can replace the  {\em coherence} condition (\ref{coherence_2}) in Definition \ref{fuzzy_cfm-model}   with  the notion of {\em $\varphi$-coherence} of an interpretation $I$ wrt the knowledge base $K$.

Observe that,  for all $x$ such that $C_i(x)>0$, condition (\ref{fi_coherence}) above corresponds to condition $C_i^I(x)= \varphi (W_i(x))$.
While the notions of {\em coherence} and of weight $W_i(x)$ (of an element $x$ wrt a concept $C_i$) consider, as a special case, the case when $C_i(x)=0$, in condition (\ref{fi_coherence}) we impose the same constraint to all domain elements $x$ (including those with $C_i(x)=0$).

For Multilayer Perceptrons, let us associate a concept name $C_i$ to each unit $i$ in a deep network $\enne$, and let us interpret,  as in \cite{JELIA2021}, a synaptic connection between neuron $h$ and neuron $i$ with weight $w_{ih}$ as the conditional $\tip(C_i) \sqsubseteq C_j$ with weight $w_h^i=w_{ih}$. If we assume that $\varphi$ is the {\em activation function} of {\em all units} in the network $\enne$, then condition (\ref{fi_coherence}) characterizes the stationary states of MLPs, where 
$C_i^I(x)$ corresponds to the activation of neuron $i$ for some input stimulus $x$ and $\sum_{h} w_h^i  \; D_{i,h}^I(x)$ corresponds to the {\em induced local field} $u_i$ of neuron $i$, which is obtained by summing the input signals to the neuron, $x_1, \ldots, x_n$, weighted by the respective synaptic weights: $u_i=\sum^n_{h=1} w_{ih} x_h$ \cite{Haykin99}, where the bias $b_i$ has been included in the sum, by adding a new synapse with input $x_0=+1$ and synaptic weight $w_{i0}=b_i$. Here, each $D_{i,h}^I(x)$ corresponds to the input signal $x_h$, for the input stimulus $x$.
Of course, {\em $\varphi$-coherence} could be easily extended to deal with a different activation functions $\varphi_i$  for each concept $C_i$ (i.e., for each unit $i$).


\begin{proposition} \label{prop:phi_coherent_models}
Let $K$ be a weighted conditional knowledge base and $\varphi: \mathbb{R} \rightarrow [0,1]$.
(1) if $\varphi$ is a {\em monotonically non-decreasing} function, a $\varphi$-coherent fuzzy multipreference model $I$ of $K$ is also an fm-model of $K$;
(2) if $\varphi$ is a {\em monotonically increasing} function, a $\varphi$-coherent fuzzy multipreference model $I$ of $K$ is also an cf$^m$-model of $K$.
\end{proposition}
\begin{proof}
For item (1),  let us assume that $\varphi$ is monotonically non-decreasing and that $I=\langle \Delta, \cdot^I \rangle$ is a $\varphi$-coherent fuzzy multipreference model of $K$. Then for all  $C_i \in {\cal C}$ and $z\in \Delta$, s.t. $C_i^I(z)>0$, $C_i^I(z)= \varphi (W_i(z))$. 
To prove condition (\ref{faitfulness}), let us assume that, for $x,y \in \Delta$ $x  <_{C_i}  y$ holds, i.e., $C_i^I(x) > C_i^I(y)$. This also implies $C_i^I(x)>0$.
If $C_i^I(y)=0$, $W_i(y)= - \infty$, and the thesis follows.
If $C_i^I(y)>0$,  both the equalities $C_i^I(x)= \varphi (W_i(x))$ and $C_i^I(y)= \varphi (W_i(y))$ hold.
 Suppose by absurd  that $W_i(x) > W_i(y)$ does not hold, i.e.,  that $W_i(x) \leq W_i(y)$.
As $\varphi$ is monotonically non-decreasing, $\varphi(W_i(x)) \leq \varphi (W_i(y))$.
Hence, by the equalities above, $C_i^I(x) \leq C_i^I(y) $, 
which contradicts the assumption that $C^I(x) > C^I(y)$.

For item (2), assume that $\varphi$ is monotonically increasing and that $I=\langle \Delta, \cdot^I \rangle$ is a $\varphi$-coherent fuzzy multipreference model of $K$. Then, for all  $C_i \in {\cal C}$ and $z\in \Delta$, $C_i^I(z)>0$ $C_i^I(z)= \varphi (W_i(z))$. 
We have to prove that condition (\ref{coherence_2}) holds, i.e., that
$x  <_{C_i}  y$ iff $ W_i(x) > W_i(y)$. The ``only if" direction holds with the same proof as for item (1), as $\varphi$ is also monotonically non-decreasing.
To prove the ``if" direction, assume that  $ W_i(x) > W_i(y)$ holds, for some $x,y \in \Delta$. 
If  $W_i(y)= - \infty$, it must be that $C_i^I(y)=0$ and $C_i^I(x)>0$, and hence $C^I(x) > C^I(y)$ follows.
If $W_i(y)\neq - \infty$, $C_i^I(y)>0$ and $C_i^I(x)>0$.
From $ W_i(x) > W_i(y)$, as $\varphi$ is monotonically increasing,
$ \varphi (W_i(x) ) > \varphi(W_i(y))$. Hence, $C^I(x) > C^I(y)$.
\qed
\end{proof}
Item 2 can be regarded as the analog of Proposition 1 in \cite{JELIA2021,arXiv_JELIA2020}, where the fuzzy multi-preferential interpretation $\emme_{\enne}^{f,\Delta}$ of a deep neural network $\enne$ built over the domain of input stimuli $\Delta$  is proven to be a coherent model of the knowledge base $K^{\enne}$ associated to $\enne$, under the specified conditions on the activation function $\varphi$, and the assumption that each stimulus in $\Delta$ corresponds to a stationary state in the neural network. Item 1 in Proposition \ref{prop:phi_coherent_models} is as well the analog of Proposition 2 in \cite{arXiv_JELIA2020} stating that  $\emme_{\enne}^{f,\Delta}$ is a faithful (or weakly-coerent)  model  of $K^{\enne}$.



A notion of {\em coherent/faithful/$\varphi$-coherent multipreference entailment} from a weighted $\lcFt$ knowledge base $K$ can be defined in the obvious way (see \cite{JELIA2021,ECSQARU2021} for the definitions of  coherent and faithful (fuzzy) multipreference entailment).
In \cite{ECSQARU2021} the properties of faithful entailment are studied.
Faithful entailment is reasonably well-behaved: it deals with specificity and irrelevance; it is not subject to inheritance blocking;
it satisfies most KLM properties \cite{KrausLehmannMagidor:90,whatdoes}, depending on their fuzzy reformulation and  on the chosen combination functions. 

As MLPs are usually represented as a weighted graphs \cite{Haykin99}, whose nodes are units and whose edges are the synaptic connections between units with their weight, it is very tempting to extend the different semantics of weighted knowledge bases considered above, 
to weighted argumentation graphs.


\section{Coherent, faithful and $\varphi$-coherent semantics for weighted argumentation graphs} \label{sec:varphi_labellings}

There is much work in the literature concerning extension of Dung's argumentation framework \cite{Dung95} with weights attached to arguments and/or to the attacks/supports between arguments. 
Many different proposals have been investigated and compared in the literature. Let us mention \cite{CayrolJAIR2005,Janssen2008,Dunne2011,Leite2013,Amgoud2017,BaroniRagoToni2018,Amgoud2019,VesicECSQARU21}, 
which also include extensive comparisons. 
In the following, we will propose some semantics for weighted argumentation with the purpose of establishing some links to the semantics of conditional knowledge bases considered in the previous sections.

In the following, we will consider a notion of {\em weighted argumentation graph} as a triple 
$G=\la {\cal A}, {\cal R}, \pi \ra$, where ${\cal A}$ is a set of arguments, ${\cal R} \subseteq {\cal A} \times {\cal A}$ 
and $\pi: {\cal R} \rightarrow  \mathbb{R}$.
This definition of weighted argumentation graph corresponds to the definition of {\em weighted argument system} in \cite{Dunne2011}, but here we admit both positive and negative weights, while  \cite{Dunne2011} only allows for positive weights representing the strength of attacks. In our notion of weighted graph, a pair $(A,B) \in {\cal R}$ can be regarded as a {\em support} relation  when the weight is positive and an {\em attack} relation when the weight is negative, and it leads to bipolar argumentation \cite{AmgoudCayrol2004}. 
The argumentation semantics we will consider in the following, as in the case of weighted conditionals, deals with both the positive and the negative weights in a uniform way. 
For the moment we do not include in $G$ a function determining the {\em basic strength} of arguments \cite{Amgoud2017}.

Given a weighted argumentation graph $G=\la {\cal A}, {\cal R}, \pi \ra$, we define a {\em labelling of the graph $G$} as a function $\sigma: {\cal A} \rightarrow [0,1]$
which assigns to each argument and {\em acceptability degree}, i.e., a value in the interval $[0,1]$. 
Let $\mathit{ R^{-}(A)=\{B \mid (B,A) \in R}\}$. When $\mathit{R^{-}(A)= \emptyset}$, argument $A$ has neither supports nor attacks.

For a weighted graph $G=\la {\cal A}, {\cal R}, \pi \ra$ and a labelling $\sigma$, we introduce a  {\em weight $W^G_{\sigma}$ on ${\cal A}$}, as a partial function $W^G_\sigma : {\cal A} \rightarrow  \mathbb{R} $, assigning a positive or negative support to the arguments $A_i \in {\cal A}$ such that $\mathit{R^{-}(A_i) \neq \emptyset}$, as follows:
 \begin{align}\label{weight_of_arguments}
	W^G_\sigma(A_i)   = \sum_{(A_j,A_i) \in {\cal R}} \pi(A_j,A_i)  \; \sigma(A_j)
\end{align} 
When $\mathit{R^{-}(A_i) = \emptyset}$,  $W^G_\sigma(A_i)$ is let undefined.

We can now exploit this notion of weight of an argument to define different argumentation semantics for a graph $G$ as follows.

\begin{definition}
Given a weighted graph $G=\la {\cal A}, {\cal R}, \pi \ra$ and a labelling $\sigma$:
\begin{itemize}
\item
$\sigma$ is a {\em coherent labelling of $G$} 
if, for all arguments $A,B \in {\cal A}$ s.t. $\mathit{R^{-}(A)\neq \emptyset}$ and $\mathit{ R^{-}(B)\neq \emptyset}$,
$$\sigma(A) < \sigma(B) \;  \iff \; W^G_\sigma(A) < W^G_\sigma(B); $$

\item
$\sigma$ is a {\em faithfull labelling of $G$} 
if, for all arguments $A,B \in {\cal A}$ s.t. $\mathit{R^{-}(A)\neq \emptyset}$ and $\mathit{ R^{-}(B)\neq \emptyset}$,
$$\sigma(A) < \sigma(B) \;  \Rightarrow \; W^G_\sigma(A) < W^G_\sigma(B); $$

\item
for a function $\varphi: \mathbb{R} \rightarrow [0,1]$,
$\sigma$ is a {\em $\varphi$-coherent labelling of $G$} 
if, for all arguments $A\in {\cal A}$ s.t. $\mathit{ R^{-}(A) \neq \emptyset}$, 
$\sigma(A) =  \varphi(W^G_\sigma(A))$.
\end{itemize}
\end{definition}
These definitions do not put any constraint on the labelling of arguments which do not have incoming edges in $G$: 
their labelling is arbitrary, provided the constraints on the labelings of all other arguments can be satisfied, depending on the semantics considered.

The definition of $\varphi$-coherent labelling of $G$ is defined through a set of equations, as in Gabbay's equational approach to argumentation networks \cite{Gabbay2012}. 
Here, we use  equations for defining the weights of arguments starting from the weights of attacks/supports.  

A $\varphi$-coherent labelling of a weigthed graph $G$ can be proven to be as well a coherent labelling or a faithful labelling, under some conditions on the function $\varphi$.
\begin{proposition} \label{prop:phi_coherent_labellings}
Given a weighted graph $G=\la {\cal A}, {\cal R}, \pi \ra$:
(1) A coherent labelling of $G$ is a faithful labelling of $G$;
(2) if $\varphi$ is a {\em monotonically non-decreasing} function, a $\varphi$-coherent labelling $\sigma$ of $G$ is  a faithful labelling of $G$;
(3) if $\varphi$ is a {\em monotonically increasing} function,  a $\varphi$-coherent labelling $\sigma$ of $G$ is  a coherent labelling of $G$.
\end{proposition}
\begin{proof}
The proof is similar to the one of Proposition \ref{prop:phi_coherent_models}, and exploits the property of a $\varphi$-labelling that 
$\sigma(A) =  \varphi(W^G_\sigma(A))$, for all arguments $A$ with $\mathit{ R^{-}(A)\neq \emptyset}$, as well as the properties of $\varphi$.

Item (1) is obvious from the definitions of coherent and faithful labellings.
For item (2),  let us assume that $\varphi$ is monotonically non-decreasing and that $\sigma$ is $\varphi$-coherent labelling of $G$. Then, for all  $C_i \in {\cal C}$ and $A \in {\cal A}$ sucht that $\mathit{ R^{-}(A) \neq \emptyset}$,  $\sigma(A) =  \varphi(W^G_\sigma(A))$. We have to prove that  for all arguments $A,B \in {\cal A}$ s.t. $\mathit{R^{-}(A)\neq \emptyset}$ and $\mathit{ R^{-}(B)\neq \emptyset}$,
$$\sigma(A) < \sigma(B) \;  \Rightarrow \; W^G_\sigma(A) < W^G_\sigma(B). $$
If $\sigma(A) < \sigma(B)$ holds, from the hypothesis that $\mathit{R^{-}(A)\neq \emptyset}$ and $\mathit{ R^{-}(B)\neq \emptyset}$, we have 
$ \varphi(W^G_\sigma(A)) <  \varphi(W^G_\sigma(B))$.
By absurd, assume that $W^G_\sigma(A) < W^G_\sigma(B)$ does not hold; then $W^G_\sigma(A) \geq W^G_\sigma(B)$. As $\varphi$ is non-decreasing, 
$ \varphi(W^G_\sigma(A)) \geq  \varphi(W^G_\sigma(B))$ holds, a contradiction.

For item (3), let us assume that $\varphi$ is monotonically increasing and that $\sigma$ is $\varphi$-coherent labelling of $G$. Then, for all  $C_i \in {\cal C}$ and $A \in {\cal A}$ sucht that $\mathit{ R^{-}(A) \neq \emptyset}$,  $\sigma(A) =  \varphi(W^G_\sigma(A))$. We have to prove that  for all arguments $A,B \in {\cal A}$ s.t. $\mathit{R^{-}(A)\neq \emptyset}$ and $\mathit{ R^{-}(B)\neq \emptyset}$,
$$\sigma(A) < \sigma(B) \;  \iff \; W^G_\sigma(A) < W^G_\sigma(B). $$
The ``$\Rightarrow$" direction holds with the same proof as for item (1), as $\varphi$ is monotonically non-decreasing.
For the  ``$\Leftarrow$" direction, let $W^G_\sigma(A) < W^G_\sigma(B)$.
As $\varphi$ is monotonically increasing, $\varphi(W^G_\sigma(A)) < \varphi(W^G_\sigma(B))$ holds. 
As $\mathit{R^{-}(A)\neq \emptyset}$ and $\mathit{ R^{-}(B)\neq \emptyset}$, since $\sigma$ is a $\varphi$-coherent labelling of $G$, the equivalences 
$\sigma(A) =  \varphi(W^G_\sigma(A))$ and $\sigma(B) =  \varphi(W^G_\sigma(B))$ hold and, hence,
$\sigma(A) < \sigma(B)$ holds.
\qed
\end{proof}
\normalcolor

\section{$\varphi$-coherent labellings and the gradual semantics}  \label{sec:gradual_sem}

The notion of $\varphi$-coherent labelling relates to the framework of gradual semantics studied by Amgoud and Doder  \cite{Amgoud2019} where, for the sake of simplicity, the weights of arguments and attacks are in the interval $[0,1]$. Here, as we have seen, positive and negative weights are admitted to represent the strength of attacks and supports. To define an evaluation method for $\varphi$-coherent labellings,  we need to consider a slightly  extended definition of an evaluation method for a graph $G$  in \cite{Amgoud2019}. Following  \cite{Amgoud2019} we include a function $\sigma_0: {\cal A} \rightarrow [0,1]$ in the definition of a weighted graph, where $\sigma_0$ assigns to each argument $A \in {\cal A}$ its basic strength.  Hence a graph $G$ becomes a quadruple  $ G=\la {\cal A}, \sigma_0, {\cal R}, \pi \ra$. 

An {\em evaluation method} for a graph $ G=\la {\cal A}, \sigma, {\cal R}, \pi \ra$ is a triple $M=\la h,g,f\ra$, where\footnote{This definition is the same as in \cite{Amgoud2019}, but for the fact that in the domain/range of functions $h$ and $g$ interval  $[0,1]$ is sometimes replaced by  $\mathbb{R}$.}:
\begin{quote}
$h:  \mathbb{R} \times [0,1] \rightarrow \mathbb{R} $ \\
$g:  \bigcup_{n=0}^{+\infty} \mathbb{R}^n \rightarrow \mathbb{R} $ \\
$h: [0,1] \times Range(g) \rightarrow [0,1]$
\end{quote}
 Function $h$ is intended to calculate the strength of an attack/support by aggregating the weight on the edge between two arguments 
 with the strength of the attacker/supporter.
Function $g$ aggregates the strength of all attacks and supports to a given argument, and function $f$ returns a value for an argument, given the strength of the argument and aggregated weight of its attacks and supports. 

As in \cite{Amgoud2019}, a gradual semantics $S$ is a function assigning to any graph $ G=\la {\cal A}, \sigma_0, {\cal R}, \pi \ra$ a weighting  $Deg^S_G$ on ${\cal A}$, i.e., $Deg^S_G : {\cal A} \rightarrow  [0,1] $, where  $Deg^S_G (A)$ represents the strength of an argument $A$ (or its acceptability degree).

A {\em gradual semantics $S$} is {\em based on an evaluation method $M$} iff, $\forall \; G=\la {\cal A}, \sigma_0, {\cal R}, \pi \ra$, $\forall A \in {\cal A}$,
 \begin{align}\label{Deg_Amgoud}
Deg^S_G (A) = f(\sigma_0(A), g(h(\pi(B_1,A), Deg^S_G (B_1) ), \ldots, h(\pi(B_n,A), Deg^S_G (B_n) ))
\end{align} 
where $\mathit{B_1, \ldots, B_n}$ are all arguments attacking or supporting $A$ (i.e., $\mathit{R^-(A)=\{B_1, \ldots,}$ $\mathit{ B_n\}}$).

Let us consider the evaluation method $M^\varphi =\la h_{prod},g_{sum},f_{\varphi} \ra$, where  the functions $h_{prod}$ and $g_{sum}$ are defined as in  \cite{Amgoud2019}, i.e.,  $h_{prod}(x,y)=x \cdot y$ and $g_{sum}(x_1, \ldots, x_n)= \sum_{i=1} ^n x_i$, but we let $g_{sum}()$ to be {\em undefined}.
We let $f_{\varphi} (x,y)= x$ when $y$ is undefined, and $f_{\varphi} (x,y)= \varphi(y)$ otherwise.
The function $f_{\varphi}$ returns a value which is independent from the first argument, when the second argument is not undefined (i.e., there is some support/attack for the argument).
When $A$ has neither attacks nor supports ($\mathit{ R^{-}(A)= \emptyset}$), $f_{\varphi}$ returns the basic strength of $A$, $\sigma_0(A)$.

The evaluation method $M^\varphi =\la h_{prod},g_{sum},f_{\varphi} \ra$ provides a characterization of the $\varphi$-coherent labelling for an argumentation graph, in the following sense.

\begin{proposition}
Let $G=\la {\cal A}, {\cal R}, \pi \ra$ be a weighted argumentation graph.
If, for some $\sigma_0: {\cal A} \rightarrow [0,1]$, $S$ is a gradual semantics of graph $G'=\la {\cal A}, \sigma_0, {\cal R}, \pi \ra$ based on the evaluation method $M^\varphi =\la h_{prod},g_{sum},f_{\varphi} \ra$,
then 
$Deg^S_{G'}$ is a $\varphi$-coherent labelling for $G$. 

Vice-versa, if $\sigma$ is a $\varphi$-coherent labelling for $G$, 
then there are a function $\sigma_0$ and a gradual semantics $S$ based on the evaluation method $M^\varphi =\la h_{prod},g_{sum},f_{\varphi} \ra$, such that, for the graph  $G'=\la {\cal A}, \sigma_0, {\cal R}, \pi \ra$, 
$Deg^S_{G'} \equiv \sigma$. 
\end{proposition}

\begin{proof} 
Let S be gradual semantics based on the method $M^\varphi =\la h_{prod},g_{sum},f_{\varphi} \ra$.
It is easy to see that, for any graph $G'=\la {\cal A}, \sigma_0, {\cal R}, \pi \ra$ and 
for any $A_h \in {\cal A}$ 
 \begin{align}\label{Deg}
	Deg^S_{G'} (A_h) & = \left\{\begin{array}{ll}
						  \varphi ( \sum_{i=1} ^n   \pi(A_i,A_h) \cdot Deg^S_{G'}(A_i) ) & \mbox{ \ \ \ \  if $\mathit{ R^{-}(A_h) \neq \emptyset}$}\\
						 \sigma_0(A_h) &  \mbox{ \ \ \ \  otherwise }  
					\end{array}\right.
\end{align} 
where $\mathit{ R^{-}(A_h) =\{A_1, \ldots, A_n\}}$. 
$Deg^S_{G'} (A_h)$ does not depend on $\sigma_0$ when $\mathit{ R^{-}(A_h) \neq \emptyset}$.  

A $\varphi$-coherent labellings of $G=\la {\cal A}, {\cal R}, \pi \ra$ is as well defined by system of equations: for all $A_h \in {\cal A}$ s.t. $\mathit{ R^{-}(A_h) \neq \emptyset}$,
$\sigma(A_h) =  \varphi(W^G_\sigma(A_h))$, i.e., by replacing $W^G_\sigma(A_h)$ with its definition,
\begin{quote}
$\sigma(A_h) =  \varphi (  \sum_{(A_i,A_h) \in {\cal R}} \pi(A_i,A_h)  \; \sigma(A_i) )$.
\end{quote}
Hence,
\begin{quote}
$\sigma(A_h) =  \varphi ( \sum_{i=1} ^n   \pi(A_i,A_h) \cdot \sigma(A_i) )$, \ \  for all $A_h \in {\cal A}$ with $\mathit{ R^{-}(A_h) \neq \emptyset}$
\end{quote}
where $\mathit{ R^{-}(A_h) =\{A_1, \ldots, A_n\}}$. 
As there is no constraints on the labelling of the arguments $A_h$ such that $\mathit{ R^{-}(A_h) = \emptyset}$,  we can chose
\begin{quote}
$\sigma(A_h)= \sigma_0(A_h)$, \ \  for all $A_h \in {\cal A}$ with $\mathit{ R^{-}(A_h) = \emptyset}$
\end{quote}
The two systems of equations above have the same solutions.
The solutions $v_{A_h}=v^{*}_{A_h}$ of the system of equations
\begin{quote}
$v_{A_h} =  \varphi ( \sum_{i=1} ^n   \pi(A_i,A_h) \cdot v_{A_i} )$, \ \  for all $A_h \in {\cal A}$ with $\mathit{ R^{-}(A_h) \neq \emptyset}$\\
$v_{A_h}= \sigma_0(A_h)$ \ \ \ \ \ \ \ \ \ \ \ \ \ \ \ \ \ \ \ \ \ \ \ \ \ \ \ \ \ \ \ \ for all $A_h \in {\cal A}$ with $\mathit{ R^{-}(A_h) = \emptyset}$

\end{quote}
correspond to the semantics S, based on $M^\varphi$, with $Deg^S_{G'}(A_h)=v^{*}_{A_h}$, as well as to the $\varphi$-coherent labelling $\sigma$ of $G$, where $\sigma(A_k)= v^{*}_{A_h}$.
\qed
\end{proof}
Amgoud and Doder  \cite{Amgoud2019} study a large family of {\em determinative} and {\em well-behaved} evaluation models for weighted graphs in which attacks have positive weights in the interval $[0,1]$.
For weighted graph $G$  with positive and negative weights, 
the  evaluation method $M^\varphi$ cannot be guaranteed to be determinative,
even under the conditions that $\varphi$ is  monotonically increasing and continuous. 
In general, there is not a unique semantics $S$ based on $M^\varphi$, and there is not a unique $\varphi$-coherent labelling for a weighted graph $G$, given a basic strength $\sigma_0$.
  This is not surprising, considering that  $\varphi$-coherent labelings of a graph correspond to stationary states (or equilibrium states) \cite{Haykin99} in a deep neural network. 

A deep neural network can than be seen as a weighted argumentation graph, with positive and negative weights, where each unit in the network is associated to an argument, and the activation value of the unit can be regarded as the weight (in the interval $[0,1]$) of the corresponding argument.
Synaptic positive and negative weights correspond to the strength of supports (when positive) and attacks (when negative).
In this view, $\varphi$-coherent labelings, assigning to each argument a weight in the interval $[0,1]$, correspond to stationary states of the network, 
the solutions of a set of equations.
This is in agreement with previous results on the relationship between weighted argumentation graphs and MLPs established by Garcez, Gabbay and Lamb \cite{GarcezArgumentation2005} and, more recently, by Potyca \cite{PotykaAAAI21}. 
We refer to the conclusions for a comparison with these approaches.

Unless the network is feedforward (and the corresponding graph is acyclic), stationary states cannot be uniquely determined by an iterative process from the values of input units (that is, from an initial labelling $\sigma_0$).
On the other  hand, a semantics $S$ based on $M^\varphi$ satisfies some of the properties considered in \cite{Amgoud2019},  including {\em anonymity, independence, directionality, equivalence} and {\em maximality}, provided the last two properties are properly reformulated to deal with both positive and negative weights (i.e., by replacing $\mathit{ R^{-}(x)}$ to $Att(x)$, for each argument $x$ in the formulation in \cite{Amgoud2019}).  
However, a semantics $S$ based on $M^\varphi$ cannot be expected to satisfy the properties of {\em neutrality,  weakening, proportionality} and {\em resilience}. In fact, function $f_{\varphi}$ completely disregard the initial valuation $\sigma_0$ in graph $G=\la {\cal A}, \sigma_0, {\cal R}, \pi \ra$, for those arguments having some incoming edge (even if their weight is $0$). So, for instance,  it is not the same, for an argument to have a support with weight $0$ or no support or attack at all: {\em neutrality} does not hold.


\normalcolor

\section{Back to conditional interpretations} \label{sec:Back} 

An interesting question is whether, given a set of possible labelings $\Sigma= \{ \sigma_1, \sigma_2, \ldots \}$ for a weighted argumentation graph $G$, where each labelling $\sigma_i$ assigns to each argument a value in the interval $[0,1]$ with respect to a given semantics, one can define a preferential structure starting from $\Sigma$ to evaluate conditional properties of the argumentation graph. This would allow, for instance, to verify properties like: "does normally argument $A_2$ follows from argument $A_1$ with a degree greater than $0.7$?" This query can be formalized by the fuzzy inclusion $\tip(A_1) \sqsubseteq A_2 > 0.7$.

In particular, let $\Sigma$ is a {\em finite set of} $\varphi$-coherent labelings $\sigma_1, \sigma_2, \ldots$ of a weighted graph $G=\la \mathcal{A}, \mathcal{R}, \pi \ra$, for some function $\varphi$ as defined in Section  \ref{sec:varphi_labellings}.
One can define a fuzzy multipreference interpretation over $\Sigma$ 
by adopting the construction used in \cite{JELIA2021} to build a fuzzy multipreference interpretation over the set of input stimuli of a neural network, where each input stimulus was associated to a fit vector \cite{Kosko92} describing the activity levels of all units for that input. Here, each labelling $\sigma_i$ plays the role of a fit vector and each argument $A$ in $\mathcal{A}$ 
can be interpreted as a concept name of the language. Let  $N_C= \mathcal{A}$ and $N_I=\{x_1,x_2, \ldots\}$. We assume that there is one individual name $x_j$ in the language for each labelling $\sigma_j \in \Sigma$,
and define a fuzzy multipreference interpretation $I^G_\Sigma=\la \Sigma, \cdot^I)$ as follows:
\begin{itemize}
\item
for all $x_j \in N_I$, $x_j^I=\sigma_j$;
\item
for all $A \in N_C$, $A^I(\sigma_j)= \sigma_j(A)$.
\end{itemize} 
The fuzzy $\lc$ interpretation $I^G_\Sigma$ induces a preference relation $<_{A_i}$ for each argument $A_i \in \mathcal{A} $. For all $\sigma_j, \sigma_k \in \Sigma$:
\begin{center}
$\sigma_j <_{A_i} \sigma_k$\ \  iff \ \ $A_i^I(\sigma_j) > A_i^I(\sigma_k)$ \ \  iff \ \  $ \sigma_j(A_i) > \sigma_j(A_i) $.
\end{center}
Let $K^G$ be the conditional knowledge base extracted from the weighted argumentation graph, as follows:
 $$K^G= \{ \tip(A_i) \sqsubseteq A_j \mid (A_j,A_i) \in \mathcal{R} \mbox{ and } w((A_j,A_i))=w_{ji}\}$$
It can be proven that:

\begin{proposition} \label{prop:preferential_argumentation_semantics}
Let $\Sigma$ is a finite set of $\varphi$-coherent labelings  of a weighted graph $G=\la \mathcal{A}, \mathcal{R}, \pi \ra$, for some function $\varphi: \mathbb{R} \rightarrow [0,1]$. The following statements hold:
\begin{itemize}
\item[(i)]
If $\varphi$  is a monotonically increasing function and $\varphi: \mathbb{R} \rightarrow (0,1]$, then
 $I^\Sigma$ is a coherent (fuzzy) multipreference model of $K^G$.
 \item[(ii)]
 If  $\varphi$ is a monotonically non-decreasing function, then
 $I^\Sigma$ is a faithful (fuzzy) multipreference model of $K^G$.
 \end{itemize} 
\end{proposition}
The proof of item (i) is similar to the proof of Proposition 1 in  \cite{JELIA2021} (the proof can be found in \cite{arXiv_JELIA2020}).
The proof of item (ii) is similar to the proof of Proposition 2 in  \cite{arXiv_JELIA2020}.
The restriction to a {\em finite set} $\Sigma$ of $\varphi$-coherent labelings is needed to guarantee the well-foundedness of the resulting interpretation.
In fact, in general, the set of all of $\varphi$-coherent labelings might be infinite and there is no guarantee that the resulting fuzzy $\lcFt$ interpretation is witnessed and the preference relations $<_{A_i} $ is well-founded.

An interesting question is whether this result can be extended to other gradual semantics,  based on a similar construction, and  under which conditions on the evaluation method.
If we consider a gradual semantics $S$ of a graph $G$, based on an evaluation method $M=\la h,g,f\ra$ according to \cite{Amgoud2019}, but extended as done in Section \ref{sec:gradual_sem}  with positive and negative real-valued weights for the pairs in $\mathcal{R}$,
the set $\Sigma$ could be taken to be a (finite) set of the solutions for $Deg_G^S$, satisfying the system of equations (\ref{Deg_Amgoud}). 
In particular, one could investigate which conditions  on the evaluation method $M=\la f,g,h \ra$ 
 guarantee that one can define a coherent or a faithful multipreference model over a finite set of labelings $\Sigma$ of a graph $G$, for a given semantics $S$. This will be subject of future work.

Observe also that, in the conditional semantics in Sections  \ref{sec:closure} and \ref{sec:varphi_coherent_models}, in a typicality inclusion $\tip(C) \sqsubseteq D$,   concepts $C$ and $D$ are not required  to be concept names, but they can be complex concepts. In particular, in the fragment $\lc$ of $\alc$ we have considered in this paper, $D$ can be any boolean combination of concepts.
the correspondence between weighted conditionals $\tip(A_i) \sqsubseteq A_j$ in $K^G$ and weighted attacks/supports in the argumentation graph $G$, suggests a possible generalization of the structure of the weighted argumentation graph by allowing attacks/supports by a boolean combination of arguments.
The labelling of arguments in the set $[0,1]$ can indeed be extended to boolean combinations of arguments using the fuzzy combination functions, as for boolean concepts in the conditional semantics (e.g., by letting $\sigma(A_1 \wedge A_2)= min\{\sigma(A_1), \sigma(A_2)\}$, using the t-norm in Zadeh fuzzy logic (see Section  \ref{sec:ALC}).

The idea of allowing an argument to be attacked/supported by a boolean combination of arguments also appears to have some relationships with Abstract Dialectical Frameworks (ADFs) \cite{Brewka2013}. In an ADF $D$ "an acceptance function $C_s$ defines the cases when the statement $s$ can be accepted, depending on the acceptance status of its parents in $D$", and an acceptance function can be expressed by an acceptance condition, a boolean formula.
Two-valued and three-valued semantics have been developed for ADFs  \cite{Brewka2013},
and the relationships between ADFs and Conditional logics have already been studied by Heyninck et al. in \cite{IsbernerFLAIRS2020} through several different translations.

The above extension also relates to the work considering ``sets of attacking (resp. supporting) arguments''; i.e., several argument together attacking (or supporting) an argument.
In fact, for gradual semantics, the sets of attacking arguments framework (SETAF) 
has been studied by Yun and Vesic,
by considering ``the force of the set of attacking (resp. supporting) arguments to be the force of the weakest argument in the set" \cite{VesicECSQARU21}. 
Using t-norm in Zadeh or Goedel fuzzy logics (as above), this would correspond to interpret the set of arguments as a conjunction. 
For instance, the example from \cite{VesicECSQARU21}, where: $\mathit{hot}$ attacks $\mathit{jogging}$ with weight $0.8$;
$\mathit{rain}$ attacks $\mathit{jogging}$ with weight $0.5$; and $\mathit{hot}$ and $\mathit{rain}$ jointly support $\mathit{jogging}$ with weight $0.2$, 
could be captured by the following set of weighted conditionals:
\begin{quote}
$\mathit{ \tip(jogging) \sqsubseteq hot }$, \ \ -  0.8

$\mathit{ \tip(jogging) \sqsubseteq rain} $, \ \  - 0.5

$\mathit{ \tip(jogging) \sqsubseteq hot \sqcap rain}$,  \ \ + 0.2
\end{quote}
where here we have represented the support to argument  $\mathit{jogging}$ by the set of arguments $\{ \mathit{hot}, \mathit{rain}\}$ 
by the third conditional, which would correspond 
to have a pair  
$ (hot \wedge rain, jogging) \in \mathcal{R}$ with weight $+0.2$ in the corresponding weighted argumentation graph which allows for boolean supports/attacks.
It would be interesting to verify whether the results in Proposition  \ref{prop:preferential_argumentation_semantics} also extend to the case when an argument can be attacked/supported by a boolean combination of arguments.

\normalcolor

\section{Conclusions}

In this paper, drawing inspiration from a fuzzy preferential semantics for weighted conditionals, which has been introduced for modeling the behavior of Multilayer Perceptrons \cite{JELIA2021}, we develop some semantics for weighted argumentation graphs, where positive and negative weights can be associated to pairs of arguments. 
In particular, we introduce the notions of coherent/faithful/$\varphi$-coherent labelings, and establish some relationships among them. 
While in  \cite{JELIA2021} a deep neural network is mapped to a weighted conditional knowledge base, a deep neural network can as well be seen as a weighted argumentation graph, with positive and negative weights, under the proposed semantics. In this view, 
$\varphi$-coherent labelings correspond to stationary states in the network (where each unit in the network is associated to an argument and the activation value of the unit can be regarded as the weight of the corresponding argument).
This is in agreement with previous work on the relationship between argumentation frameworks and neural networks 
first investigated by Garcez, et al. \cite{GarcezArgumentation2005} and recently by Potyca \cite{PotykaAAAI21}.

The work by Garcez, et al. combines value-based argumentation frameworks \cite{Bench-Capon03} and neural-symbolic learning systems by providing a translation from argumentation networks to neural networks with 3 layers (input, output layer and one hidden layer). This enables the accrual of arguments through learning as well as the parallel computation of arguments.
The work by Potyca \cite{PotykaAAAI21} considers a quantitative bipolar argumentation frameworks (QBAFs) similar to \cite{BaroniRagoToni2018} 
and exploits an {\em influence function} based on the logistic function to define an MLP-based semantics $\sigma_\mathit{MLP}$ for a QBAF:  for each argument $a \in \mathcal{A}$,   $\sigma_\mathit{MLP}(a)= \lim_{k \rightarrow \infty} s_a^{(k)}$, when the limit exists, and is undefined otherwise; where 
$s_a^{(k)}$ is a value in the interval $[0,1]$, and $k$ represents the iteration.
The paper studies convergence conditions  both in the discrete and in the continuous case, as well as the semantic properties of MLP-based semantics.

In this work we have as well investigated the relationships between $\varphi$-coherent labelings and the gradual semantics by Amgoud and Doder \cite{Amgoud2019},
by slightly extending their definitions to deal with positive and negative weights to capture the strength of supports and of attacks.
A correspondence between the gradual semantics based on a specific evaluation method $M^{\varphi}$ and $\varphi$-coherent labelings has been established.
Differently from the Fuzzy Argumentation Frameworks by Jenssen et al \cite{Janssen2008},   where an attack relation is a fuzzy binary relation over the set of arguments, here we have considered real-valued weights associated to pairs of arguments.

The paper suggests an approach to deal with 
attack/supports by a boolean combination of arguments, by exploiting the fuzzy semantics of weighted conditionals.
Finally, the paper discusses an approach for defeasible reasoning over a weighted argumentation graph building on $\varphi$-coherent labelings. 
This allows a multipreference model to be constructed over a (finite) set of $\varphi$-labelling $\Sigma$ and allows (fuzzy) conditional formulas over arguments to be validated over $\Sigma$ by model checking over a preferential model.
Whether this approach can be extended to the other gradual semantics, and under which conditions on the evaluation method, will be subject of future work.

The correspondence between Abstract Dialectical Frameworks  \cite{Brewka2013} and Nonmonotonic Conditional Logics has been studied by Heyninck et al. in \cite{IsbernerFLAIRS2020},  
with respect to the two-valued models, the stable, the preferred semantics and the grounded semantics of ADFs.
Whether the coherent/faithfull/$\varphi$-coherent semantics developed in the paper for weighted argumentation can be reformulated for a (weighted) Abstract Dialectical Frameworks, and which are the relationships with the work in \cite{IsbernerFLAIRS2020}, also requires investigation for future work.

\color{blue}

\normalcolor

Undecidability results for fuzzy description logics with general inclusion axioms 
\cite{CeramiStraccia2011,BorgwardtPenaloza12} 
motivate restricting the logics to finitely valued semantics \cite{BorgwardtPenaloza13}, 
and also motivate the investigation of decidable approximations of fuzzy multipreference entailment, under the different semantics. 
A similar investigation is also of interest for the semantics of weighted argumentation graphs introduced in this paper.




\end{document}